\documentclass[11pt]{article}
\usepackage{fullpage}
\usepackage{natbib}
\usepackage{booktabs} 
\usepackage{multirow}
\usepackage{pgfplots}
\usepackage{amsfonts}       
\usepackage{nicefrac}       
\usepackage{graphicx,amsmath,amssymb,amsthm}
\DeclareFontShape{OT1}{cmss}{b}{n}{<->ssub * cmss/bx/n}{} 
\usepackage{algorithm,algorithmic}
\usepackage{enumitem}
\usepackage{siunitx}
\usepackage{listings}
\usepackage{xcolor}

\definecolor{codegreen}{rgb}{0,0.6,0}
\definecolor{codegray}{rgb}{0.5,0.5,0.5}
\definecolor{codepurple}{rgb}{0.58,0,0.82}
\definecolor{backcolour}{rgb}{0.92,0.92,0.92}

\newcommand{\stderr}[1]{\scriptsize$\pm$#1}

\definecolor{Blue}{RGB}{55,126,184}
\definecolor{Orange}{RGB}{217,95,2}
\definecolor{Teal}{RGB}{27,158,119}
\definecolor{Red}{RGB}{228,26,28}
\definecolor{Purple}{RGB}{117,112,179}

\lstdefinestyle{mystyle}{
    backgroundcolor=\color{backcolour},   
    commentstyle=\color{codegreen},
    keywordstyle=\color{magenta},
    stringstyle=\color{codepurple},
    basicstyle=\ttfamily\footnotesize,
    breakatwhitespace=false,         
    breaklines=true,                 
    captionpos=b,                    
    keepspaces=true,                 
    showspaces=false,                
    showstringspaces=false,
    showtabs=false,                  
    tabsize=2
}
\usepackage[labelformat=simple]{subcaption}
\renewcommand\thesubfigure{(\alph{subfigure})}
\newcommand\newsubcap[1]{\phantomcaption%
  \caption*{\figurename~\thefigure\thesubfigure: #1}}

\usepgfplotslibrary{colorbrewer}
\usepgfplotslibrary{groupplots}
\pgfplotsset{compat=1.17}
\pgfplotsset{cycle list/Set1-4} 

\usepackage{xurl}
\usepackage{hyperref}

\theoremstyle{plain}
\newtheorem{theorem}{Theorem}[section]

\newtheorem{lemma}[theorem]{Lemma}

\theoremstyle{definition}

\newtheorem{assumption}[theorem]{Assumption}
\theoremstyle{remark}

\usepackage[textsize=scriptsize,tickmarkheight=0.5em,textwidth=1.1\marginparwidth]{todonotes}
\makeatletter
 \if@todonotes@disabled
 
 \else
 
 \fi
\makeatother

\newcommand{\footnoteremember}[2]{%
    \footnote{#2}%
    \newcounter{#1}
    \setcounter{#1}{\value{footnote}}
}
\newcommand{\footnoterecall}[1]{\footnotemark[\value{#1}]}

\makeatletter
\def\blfootnote{\gdef\@thefnmark{}\@footnotetext}
\makeatother

\DeclareMathOperator*{\argmin}{arg\,min}

\newcommand{\minimize}{\mathop{\mbox{minimize}{}}}

\newcommand{\E}{\mathbf{E}}

\newcommand{\prox}{\mathbf{prox}}

\newcommand{\R}{\mathbf{R}}

\newcommand{\sgn}{\mathrm{sgn}}

\newcommand{\cW}{\mathcal{W}}
\newcommand{\dist}{\mbox{dist}}

\newcommand{\bftab}{\fontseries{b}\selectfont}

\newlength{\fullwidth}
\title{PARQ: Piecewise-Affine Regularized Quantization$^\S$}

\author{
Lisa Jin\footnoteremember{FAIR}{Meta FAIR, United States.
Emails: \{\texttt{lvj, gromovand, adefazio, linx}\}\texttt{@meta.com}.}
\qquad
Jianhao Ma\footnote{
University of Michigan, Ann Arbor, MI, United States. Email: \texttt{jianhao@umich.edu}.}
\qquad
Zechun Liu\footnote{Meta Reality Labs, United States. Email: \texttt{zechunliu@meta.com}} \\[1ex]
Andrey Gromov\footnoterecall{FAIR}
\qquad
Aaron Defazio\footnoterecall{FAIR}
\qquad 
Lin Xiao\footnoterecall{FAIR}
}
\date{}

\begin{document}
\maketitle


\begin{abstract}
We develop a principled method for quantization-aware training (QAT) of large-scale machine learning models.
Specifically, we show that \emph{convex}, piecewise-affine regularization (PAR) can effectively induce the model parameters to cluster towards discrete values. 
We minimize PAR-regularized loss functions using an aggregate proximal stochastic gradient method (AProx) and prove that it has \emph{last-iterate convergence}. 
Our approach provides an interpretation of the straight-through estimator (STE), a widely used heuristic for QAT, as the asymptotic form of PARQ.
We conduct experiments to demonstrate that PARQ obtains competitive performance on convolution- and transformer-based vision tasks.
\end{abstract}

\blfootnote{\!\!\!$^\S$Open-source PyTorch package:
\url{https://github.com/facebookresearch/parq}}

\section{Introduction}

Modern deep learning models exhibit exceptional vision and language processing capabilities, but come with excessive sizes and demands on memory and computing. 
Quantization is an effective approach for model compression, which can significantly reduce their memory footprint, computing cost, as well as latency for inference \citep[e.g.,][]{han2016compression,sze2017efficient}. 
There are two main classes of quantization methods: post-training quantization (PTQ) and quantization-aware training (QAT).
Both are widely adopted and receive extensive research---see the recent survey papers \citep{gholami2022survey,fournarakis2022quantizing} and references therein.

PTQ converts the weights of a pre-trained model directly into lower precision without repeating the training pipeline;
it thus has less overhead and is relatively easy to apply \cite{nagel2020up,cai2020zeroq,chee2024quip}.
However, it is mainly limited to 4 or more bit regimes and can suffer steep performance drops with fewer bits
\cite{yao2022zeroquant,dettmers2023case}. 
This is especially the case for transformer-based models, which prove harder to quantize \cite{bai2021binarybert,qin2022bibert} compared to convolutional architectures \cite{martinez2019training,qin2020forward}. 
On the other hand, QAT integrates quantization into pre-training and/or fine-tuning processes and can produce low-bit (especially binary) models with mild performance degradation \citep[e.g.][]{fan2021training,liu2022bit}. 

A key ingredient of QAT is the so-called straight-through estimator (STE), which was invented as a heuristic \cite{bengio2013ste, courbariaux2015binaryconnect} and has been extremely successful in practice \cite{rastegari2016xnor,hubara2018quantized,esser2019learned}. 
There have been many efforts trying to demystify the effectiveness of STE, especially through the lens of optimization algorithms
\citep[e.g.,][]{li2017deeperunderstanding, yin2018binaryrelax, yin2019understandingste, bai2019proxquant, ajanthan2021mirror, dockhorn2021demystifying, LuYuLiNia2023}.
However, significant gaps between theory and practice remain. 

In this paper, we develop a principled method for QAT based on \emph{convex} regularization and interpret STE as the asymptotic form of an aggregate proximal (stochastic) gradient method. 
The convex regularization framework admits stronger convergence guarantees than previous work and allows us to prove the \emph{last-iterate convergence} of the method.

\subsection{The Straight-Through Estimator (STE)}
\label{sec:ste}

\begin{figure}[t]
\begin{center}
\includegraphics[width=0.35\linewidth]{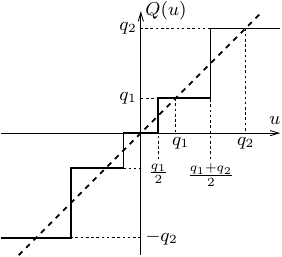}
\end{center}
\caption{A quantization map with $\mathcal{Q}=\{0,\pm q_1, \pm q_2\}$.}
\label{fig:quant-map}
\end{figure}

We consider training a machine learning model with parameters $w\in\R^d$ and let $f(w,z)$ denote the loss of the model on a training example~$z$.
Our goal is to minimize the population loss 
\(f(w) = \E_{z}[f(w,z)]\)
where~$z$ follows some unknown probability distribution.
Here we focus on the classical stochastic gradient descent (SGD) method. 
During each iteration of SGD, we draw a random training example (mini-batch) $z^t$ and update the model parameter as
\begin{equation}\label{eqn:sgd}
w^{t+1} = w^t - \eta_t \nabla f(w^t, z^t),
\end{equation}
where $\nabla f(\cdot,z^t)$ denotes the stochastic gradient with respect to the first argument (here being $w^t$) and $\eta_t$ is the step size.

QAT methods modify SGD by adding a quantization step.
In particular, the BinaryConnect method \cite{courbariaux2015binaryconnect} can be written as 
\begin{equation}\label{eqn:binary-connect}
u^{t+1} = u^t - \eta_t \nabla f(Q(u^t), z^t),
\end{equation}
where $Q(\cdot)$ is the (coordinate-wise) projection onto the set $\{\pm 1\}^d$. It readily generalizes to projection onto $\mathcal{Q}^d$ where $\mathcal{Q}$ is a finite set of arbitrary quantization values.
Figure~\ref{fig:quant-map} shows an example with $\mathcal{Q}=\{0,\pm q_1, \pm q_2\}$.

Notice that in Equation~\eqref{eqn:binary-connect} we switched notation from $w^t$ to $u^t$, because we would like to define $w^t = Q(u^t)$ as the quantized model parameters.
This reveals a key feature of QAT: the stochastic gradient in~\eqref{eqn:binary-connect} is computed at~$w^t$ instead of~$u^t$ itself
(which would be equivalent to~\eqref{eqn:sgd}).
Here we regard~$u^t$ as a full-precision (floating-point) latent variable that is used to accumulate the gradient computed at~$w^t$, and the quantization map $Q(\cdot)$ is applied to the latent variable $u^{t+1}$ to generate the next quantized variable $w^{t+1}$.

The notion of STE rises from the intent of computing an approximate gradient of the loss function with respect to~$u^t$.
Let's define the function $\tilde{f}(u,z):=f(Q(u),z)=f(w,z)$ in light of $w=Q(u)$. Then we have for each $i=1,\ldots,d$,
\[
\frac{\partial \tilde{f}}{\partial u_i} 
= \frac{\partial f}{\partial w_i}\frac{d w_i}{d u_i}
= \frac{\partial f}{\partial w_i}\frac{d Q(u_i)}{d u_i}.
\]
However, due to the staircase shape of the quantization map, we have $d Q(u_i)/d u_i=0$ and thus $\nabla \tilde{f}(u,z)=0$ almost everywhere.
STE aims to ``construct'' a nontrivial gradient with respect to~$u$, by simply treating $Q(\cdot)$ as the identity map (``straight-through'') during backpropagation, i.e., replacing $d Q(u_i)/d u_i$ with~$1$. 
This leads to the approximation
\[
\nabla \tilde{f}(u,z)~\overset{\text{STE}}{\approx}~\nabla f(w,z) = \nabla f(Q(u),z),
\]
so one can interpret Equation~\eqref{eqn:binary-connect} as an approximate SGD update for minimizing the loss $\tilde{f}(u)$.

There are several issues with the above argument. 
First, we know exactly that $d Q(u_i)/d u_i=0$ almost everywhere, so there is no need for ``approximation.''
Second, any approximation that replaces~$0$ with $1$ in this context warrants scrutiny of the resulting bias and consequences on stability.
Existing works on this are restricted to special cases and weak convergence results \citep{li2017deeperunderstanding, yin2019understandingste}.

Alternatively, we can view~\eqref{eqn:binary-connect} as an implicit algorithm for updating $w^t$ and analyze its convergence. More explicitly
\begin{equation}\label{eqn:qat}
\begin{aligned}
u^{t+1} &= u^t \!- \eta_t\, \nabla f(w^t, z^t),\\
w^{t+1} &= Q(u^{t+1}) .
\end{aligned}
\end{equation}
Here $u^t$ serves as an auxiliary variable that accumulates past gradients evaluated at $w^0,\ldots,w^t$ (similar to momentum).  
This formalism is enabled through regularization and proximal gradient methods \citep{bai2019proxquant,dockhorn2021demystifying}.
And it is the path we take in this paper.

\subsection{Outline and contributions}

In Section~\ref{sec:par}, we review the framework of regularization and introduce a family of \emph{convex}, piecewise-affine regularizers (PAR).
In addition, we derive the first-order optimality conditions for minimizing PAR-regularized functions.

In Section~\ref{sec:aprox}, we derive an aggregate proximal gradient method (AProx) for solving PAR-regularized minimization problems
 and provide its convergence analysis for convex losses. AProx applies a soft-quantization map that evolves over the iterations and asymptotically converges to hard quantization, thus giving a principled interpretation of STE.

In Section~\ref{sec:implementation}, we present PARQ
(Piecewise-Affine Regularized Quantization), 
a practical implementation of AProx  
that does not need to pre-determine the quantization values and regularization strength. 

In Section~\ref{sec:experiments}, we conduct QAT experiments on low-bit quantization of convolution- and transformer-based vision models and demonstrate that PARQ obtains competitive performance compared with STE/BinaryConnect and other methods based on nonconvex regularization.

We note that \citet{dockhorn2021demystifying}
used the regularization framework and proximal optimization to demystify the BinaryConnect algorithm~\eqref{eqn:qat} and developed a generalization called ProxConnect.
In fact, AProx is equivalent to ProxConnect albeit following quite different derivations. 
Nevertheless, we make the following novel contributions:
\vspace{-1ex}
\begin{itemize}\itemsep 0pt
\item  
We propose \emph{convex} PAR for inducing quantization.
\citet{dockhorn2021demystifying} 
focus on monotone (non-decreasing) proximal maps, which can correspond to arbitrary regularization.
Even though they present convergence results for convex regularization, no such example is given to demonstrate its relevance.
Beyond closing this gap between theory and practice, our construction of convex PAR is rather surprising --- it actually encourages clustering around discrete values.  
\item  
We derive first-order optimality conditions for minimizing PAR-regularized functions. They reveal the \emph{critical role of nonsmoothness in inducing quantization}.
\item We prove \emph{last-iterate convergence} of AProx. 
The convergence results of 
\citet{dockhorn2021demystifying} 
concern the averaged iterates generated by ProxConnect/AProx.
While such results are conventional in the stochastic optimization literature, they are far from satisfactory for QAT, because the averaged iterate may not be quantized even if every iterate is quantized. 
Last-iterate convergence gives a much stronger guarantee.
\item We propose a practical implementation called PARQ that can adaptively choose the quantization values and regularization strength in an online fashion. 
\end{itemize}

\section{Piecewise affine regularization (PAR)}
\label{sec:par}

\begin{figure}[t]
\begin{subfigure}[b]{0.5\linewidth}
\centering
\includegraphics[width=0.4\linewidth]{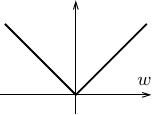}
\caption{$\Psi(w)=\|w\|_1$\qquad\mbox{}}
\label{fig:reg-L1}
\end{subfigure}%
\begin{subfigure}[b]{0.5\linewidth}
\centering
\includegraphics[width=0.6\linewidth]{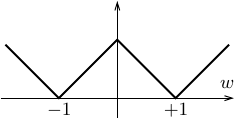}
\caption{$\displaystyle\Psi(w)=\min_{v\in\{\pm 1\}}\|w-v\|_1$\\[-1.5ex]}
\label{fig:reg-W}
\end{subfigure}%
\caption{Illustration of two nonsmooth regularizers.}
\label{fig:reg}
\end{figure}

Regularization is a common approach for inducing desired properties of machine learning models, by minimizing a weighted sum of the loss function~$f$ and a regularizer~$\Psi$: 
\begin{equation}\label{eqn:erm+reg}
\minimize_{w\in\R^d} ~~ f(w) + \lambda\Psi(w),
\end{equation}
where $\lambda\in\R_+$ is a parameter to balance the relative strength of regularization. 
It is well known that $L_2$-regularization helps generalization by preferring smaller model parameters, and $L_1$-regularization (Figure~\ref{fig:reg-L1}) induces sparsity.

There have been many attempts of using regularization to induce quantization \citep[e.g.,][]{carreiraperpinan2017compression2, yin2018binaryrelax, bai2019proxquant}.
An obvious choice is to let~$\Psi$ be the indicator function of~$\mathcal{Q}^d$; in other words, $\Psi(w)=\sum_{i=1}^d\delta_{\mathcal{Q}}(w_i)$ where
\begin{equation}\label{eqn:indicator}
\delta_{\mathcal{Q}}(w_i) = \begin{cases}
0 & \mbox{if}~w_i\in\mathcal{Q}, \\ 
+\infty & \mbox{otherwise}.
\end{cases}
\end{equation}
Then minimizing $f(w)+\lambda\Psi(w)$ is equivalent to the constrained optimization problem of minimizing $f(w)$ subject to $w\in\mathcal{Q}^d$, which is combinatorial in nature and very hard to solve in general. 
\citet{yin2018binaryrelax} propose to use the Moreau envelope of the indicator function, which under the Euclidean metric gives 
$\Psi(w)=\min_{v\in\mathcal{Q}^d}\|v-w\|_2^2$. 
A nonsmooth version is proposed by \citet{bai2019proxquant} under the $L_1$-metric, resulting in $\Psi(w)=\min_{v\in\mathcal{Q}^d}\|v-w\|_1$; Figure~\ref{fig:reg-W} shows a W-shaped example in one dimension. 

\emph{The effectiveness of a regularizer largely relies on two properties: nonsmoothness and convexity.}
Smooth regularizers such as $\dist(w,\mathcal{Q}^d)$ behave like $\|w\|_2^2$ locally, and do not induce zero or any discrete structure like hard quantization.
Nonsmooth regularizers locally behave like $\|w\|_1$ near zero; they thus tend to cluster weights towards the set of nondifferentiable points---more suitable for quantization. 

Convexity concerns the global behavior of regularization.
For example, the popularity of $L_1$-regularization for sparse optimization is largely attributed to its convexity besides being nonsmooth. 
On the other hand, it is hard for a gradient-based algorithm to cross the middle hill in the nonconvex W-shaped regularizer shown in Figure~\ref{fig:reg-W}, if the initial weights are trapped in the wrong valley from the optimal ones.
Therefore, ideally we would like to construct a regularizer that is both nonsmooth and convex.

To simplify presentation, we assume 
\(\Psi(w)=\sum_{i=1}^d \Psi(w_i)\)
and use the same notation $\Psi$ for the function of a vector or one of its coordinates (it should be self-evident from the context).  
For most of the discussion, we focus on the scalar case and omit the subscript~$i$ or simply assume $d=1$. 

\begin{figure}[t]
\centering
\includegraphics[width=0.4\linewidth]{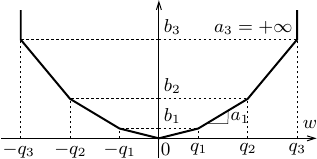}
\vspace{1ex}
\caption{Convex PAR: \(\Psi(w)=\max_{k}\{a_k(|w|-q_k)+b_k\}\).}
\label{fig:par}
\end{figure}

Suppose the set of target quantization values is given as
\(
\mathcal{Q} = \{0, \pm q_1,\ldots,\pm q_m\}\) 
and assume $0=q_0<q_1<\cdots<q_m$.
We define a piecewise-affine regularizer (PAR) as
\begin{equation}\label{eqn:par-def}
\Psi(w) = \max_{k\in\{0,\ldots,m\}} \{ a_k(|w|-q_k) + b_k \}, 
\end{equation}
where the slopes $\{a_k\}_{k=0}^m$ are free parameters that satisfy
$0 \leq a_0 < a_1 < \cdots < a_m = +\infty$, 
and $\{b_k\}_{k=0}^m$ are determined 
by setting $b_0=0$, $q_0=0$, and 
\[
b_k = b_{k-1} + a_{k-1}(q_k-q_{k-1}), \qquad k=1,\ldots,m.
\]
As shown in Figure~\ref{fig:par}, $(\pm q_k,b_k)$ are the reflection points of the piecewise-affine graph. 
The function $\Psi(w)$ is convex because the maximum of finite linear functions is convex 
\citep[Section~3.2.3]{boyd2004convex}.

We note that setting $a_0=0$ effectively removes $q_0=0$ from the quantization set~$\mathcal{Q}$ because it is no longer a reflection point of~$\Psi$.
Figure~\ref{fig:special-cases} illustrates three special cases of PAR for low-bit quantization, where both Figures~\ref{fig:special-1-bit} and~\ref{fig:special-2-bit} have $a_0=0$.
Finally, the above definition of PAR is symmetric around zero for the convenience of presentation. 
It is straightforward to extend to the asymmetric case.

\begin{figure}[t]
\captionsetup[subfigure]{aboveskip=10pt}
\centering
\captionsetup[subfigure]{position=b}
\begin{subfigure}[b]{0.33\linewidth}
\centering
\includegraphics[width=0.75\linewidth]{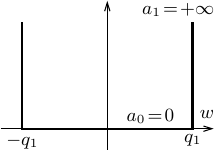}
\caption{1-bit: $\mathcal{Q}=\{\pm q_1\}$}
\label{fig:special-1-bit}
\end{subfigure}%
\begin{subfigure}[b]{0.33\linewidth}
\centering
\includegraphics[width=0.75\linewidth]{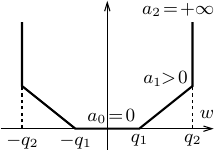}
\caption{2-bit: $\mathcal{Q}=\{\pm q_1,\pm q_2\}$}
\label{fig:special-2-bit}
\end{subfigure}
\begin{subfigure}[b]{0.33\linewidth}
\addtocounter{subfigure}{1}
\centering
\includegraphics[width=0.75\linewidth]{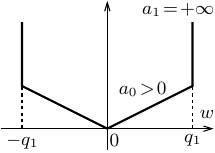}
\caption{Ternary: $\mathcal{Q}=\{0,\pm q_1\}$}
\label{fig:special-ternary}
\end{subfigure}
\caption{Three special cases of PAR for low-bit quantization.}
\label{fig:special-cases}
\end{figure}

\subsection{Optimality conditions}

In order to understand how PAR can induce quantization, we examine the optimality conditions of minimizing PAR-regularized functions.
Suppose~$f$ is differentiable and $w^\star$ is a solution to the optimization problem~\eqref{eqn:erm+reg}. 
The first-order optimality condition for this problem is
\citep[see, e.g.,][Theorem~8.18]{wright_recht2022book} 
\[
0\in\nabla f(w^\star)+\lambda\partial\Psi(w^\star),
\]
where $\partial\Psi(w^\star)$ denotes the subdifferential of $\Psi$ at $w^\star$.
For convenience, we rewrite it as
\(
\nabla f(w^\star) \in -\lambda\,\partial \Psi(w^\star),
\)
which breaks down into the following cases: 
\begin{alignat}{4}
\nonumber
w_i^\star & = -q_k, & & \Longleftarrow & \nabla_i f(w^\star) & \in \lambda\, (a_{k-1}, a_k)\\ 
\nonumber
w_i^\star & \in \!(-q_k, -q_{k-1}) ~& & \Longrightarrow ~\quad & \nabla_i f(w^\star) & = \lambda\,a_{k-1}\\ 
\nonumber
w_i^\star & = 0 & & \Longleftarrow & \!\!\!\!\!\! -\nabla_i f(w^\star) & \in \lambda\, (-a_0, a_0) \\ 
\nonumber
w_i^\star & \in (q_{k-1},q_k)\quad & & \Longrightarrow ~\quad & \nabla_i f(w^\star) & =-\lambda\,a_{k-1} \\ 
\nonumber
w_i^\star & = q_k, & & \Longleftarrow &  \nabla_i f(w^\star) & \in \lambda\, (-a_k, -a_{k-1}). 
\end{alignat}
Here the subscript~$i$ runs from~$1$ through~$d$ and~$k$ runs from~$1$ through~$m$.
The symbol  $\Longleftarrow$ ($\Longrightarrow$) means that the left-hand side expression is a necessary (sufficient) condition for the right-hand side expression.

We immediately recognize that the sufficient condition for $w_i^\star=0$ is the same as for the $L_1$-regularization $\Psi(w)=\lambda\cdot a_0\|w\|_1$. 
Further examination reveals that 
for any weight not clustered at a discrete value in~$\mathcal{Q}$, i.e., if $w_i^\star\in(q_{k-1},q_k)$, the corresponding partial derivative of~$f$ must equal to one of the $2m$ discrete values $\{\pm\lambda a_{k-1}\}_{k=1}^m$. 
Conversely, almost all values of the partial derivatives of~$f$, except for these $2m$ discrete values,
can be balanced by assigning the model parameters at the $2m+1$ discrete values in~$\mathcal{Q}$.
Intuitively, this implies that the model parameters at optimality are more likely to be clustered at these discrete values. 

\subsection{Proximal mapping of PAR}

A fundamental tool for solving problem~\eqref{eqn:erm+reg} 
is the \emph{proximal map} of the regularizer~$\Psi$, defined as
\[
\textstyle
\prox_\Psi(u) = \argmin_w \left\{\Psi(w) + \frac{1}{2}\|w-u\|_2^2\right\}.
\]
See, e.g., \citet[][\S8.6]{wright_recht2022book} for further details.
For the PAR function defined in~\eqref{eqn:par-def}, its proximal map has the following closed-form solution (let $a_{-1}=0$)
\begin{equation}\label{eqn:par-prox}
\prox_\Psi(u)
\!=\! \begin{cases}
\sgn(u) q_k & \mbox{if}~|u|\!\in\! [a_{k-1}\!+\!q_k,\;a_k\!+\!q_k], \\ 
u\!-\!\sgn(u) a_k \!\! & \mbox{if}~|u|\!\in\![a_k\!+\!q_k,\;a_k\!+\!q_{k+1}].   
\end{cases}
\vspace{-1ex}
\end{equation}
where $\sgn(\cdot)$ denote the sign or signum function. 


\begin{figure}[t]
\centering
\captionbox{Graph of \(\prox_{\Psi}(u)\).
\label{fig:par-prox}}[0.5\linewidth]{
\centering
\includegraphics[width=0.85\linewidth]{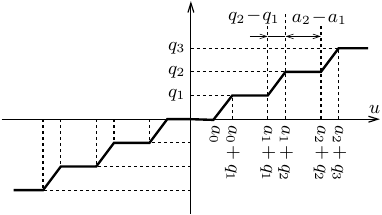}
}%
\captionbox{Graph of \(\prox_{\eta_t\lambda\Psi}(u)\).
\label{fig:par-prox-scale}}[0.5\linewidth]{
\centering
\includegraphics[width=0.85\linewidth]{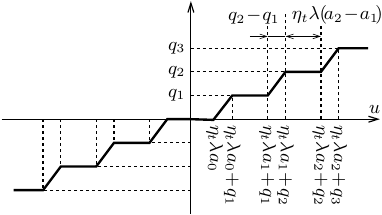}
}
\end{figure}

Figure~\ref{fig:par-prox} shows the graph of $\prox_{\Psi}(u)$, which is clearly monotone non-decreasing in~$u$.
According to \citet[][Proposition~3]{yu2015nonconvex}, a (possibly multivalued) map is a proximal map of some function if and only if it is compact-valued, monotone and has a closed graph.
For example, the hard-quantization map in Figure~\ref{fig:quant-map} is the proximal map of the (nonconvex) indicator function $\delta_\mathcal{Q}$ in~\eqref{eqn:indicator}.
\citet{dockhorn2021demystifying} work with monotone proximal maps directly without specifying the regularizer itself.
In contrast, we construct a convex regularizer, and show that (somewhat surprisingly) it can effectively induce quantization and obtain competitive performance, with stronger convergence guarantees.

\section{The Aggregate Prox (AProx) Algorithm}
\label{sec:aprox}

The regularization structure of problem~\eqref{eqn:erm+reg} can be well exploited by the \emph{proximal gradient} method 
\begin{equation}\label{eqn:prox-grad}
w^{t+1} = \prox_{\eta_t\lambda\Psi} \left(w^t - \eta_t\nabla f(w^t)\right),
\end{equation}
where $\prox_{\eta_t\lambda\Psi}$ is the proximal map of the scaled function $\eta_t\lambda\Psi$. Since $\eta_t\lambda$ effectively scales the slopes $\{a_k\}_{k=1}^m$ (with $\mathcal{Q}$ fixed), we obtain $\prox_{\eta_t\lambda\Psi}$ by simply replacing $a_k$ in~\eqref{eqn:par-prox} with $\eta_t\lambda a_k$, and 
the proximal map is shown in Figure~\ref{fig:par-prox-scale}.

If~$f$ is convex and $\nabla f$ is $L$-Lipschitz continuous, then using the constant step size $\eta_t=1/L$ leads to a convergence rate of $O(1/k)$ 
\citep[e.g.,][Theorem~9.6]{wright_recht2022book}.

In machine learning context, we have $f(w)=\E_z[f(w,z)]$ (see Section~\ref{sec:ste}). The Prox-SGD method replaces $\nabla f(w^t)$ with the stochastic gradient $g^t:=\nabla_w f(w^t,z^t)$:
\begin{equation}\label{eqn:prox-sgd}
w^{t+1} = \prox_{\eta_t\lambda\Psi} \left(w^t - \eta_t g^t \right),
\end{equation}
We assume that the step size $\eta_t$ satisfies the following classical condition to ensure convergence (with bounded $g^t$):
\begin{equation}\label{eqn:stepsize-cond}
\textstyle
\eta_t\to 0
\qquad\mbox{and}\qquad 
\sum_{t=1}^\infty \eta_t = +\infty .
\end{equation}
In this case, the flat segments on the graph of $\prox_{\eta_t\lambda\Psi}$, as shown in Figure~\ref{fig:par-prox-scale}, with lengths $\eta_t\lambda(a_k-a_{k-1})$, will all shrink to zero when $\eta_t\to 0$ (except at the two ends because $a_m=+\infty$). Therefore, the graph converges to the identity map clipped flat outside of $[-q_m,+q_m]$ and we lose the action of quantization.
This issue parallels that of using Prox-SGD with $L_1$-regularization, which does not produce sparse solutions because of the shrinking deadzone in the soft-thresholding operator as $\eta_t\to 0$ \citep{xiao2010rda}.



To overcome the problem of diminishing regularization, we derive an Aggregate Proximal gradient (AProx) method. 
Aprox shares a similar form with BinaryConnect as presented in~\eqref{eqn:qat}.
Specifically, it replaces the hard-quantization map $Q(\cdot)$ in~\eqref{eqn:qat} with an \emph{aggregate} proximal map:
\begin{equation}\label{eqn:aprox}
\begin{aligned}
u^{t+1} &= u^t - \eta_t g^t, \\
w^{t+1} &= \prox_{\gamma_t\lambda\Psi}(u^{t+1}),
\end{aligned}
\end{equation}
where $\gamma_t=\sum_{s=1}^t \eta_s$. 
Here $\prox_{\gamma_t\lambda\Psi}$ is called an aggregate map because $\lambda\Psi$ is scaled by the aggregate step size~$\gamma_t$.
In fact, BinaryConnect is a special case of AProx with $\Psi$ being the indicator function of $\mathcal{Q}^d$ given in~\eqref{eqn:indicator}.
The indicator function and its proximal map (Figure~\ref{fig:quant-map}) is invariant under arbitrary scaling, thus hiding the subtlety of aggregation.

The graph of $\prox_{\gamma_t\lambda\Psi}$ can be obtained by replacing~$\eta_t$ in Figure~\ref{fig:par-prox-scale} with $\gamma_t$.
However, according to~\eqref{eqn:stepsize-cond}, we have
\[
\textstyle
\gamma_t=\sum_{s=1}^t\eta_s\to+\infty,
\]
which implies that the flat segments in the graph, now with lengths $\gamma_t\lambda(a_k-a_{k-1})$, grow larger and larger, which is \emph{opposite} to the Prox-SGD method.

For the ease of visualization, we rescale the input~$u$ by $\gamma_t^{-1}$ and obtain the graph in Figure~\ref{fig:prox-scale-gamma}.
In this scaled graph, the lengths of the flat segments $\lambda(a_k-a_{k-1})$ stay constant but the sloped segments, with lengths $\gamma_t^{-1}(q_k-q_{k-1})$, shrink as $\gamma_t$ increases.
Asymptotically, as $\gamma_t\to \infty$, the graph converges to hard quantization, as shown in Figure~\ref{fig:prox-scale-inf}.



\begin{figure}[t]
\centering
\captionbox{$\prox_{\gamma_t\lambda\Psi}(u)$ with scaled input.
\label{fig:prox-scale-gamma}}[0.5\linewidth]{
\centering
\includegraphics[width=0.85\linewidth]{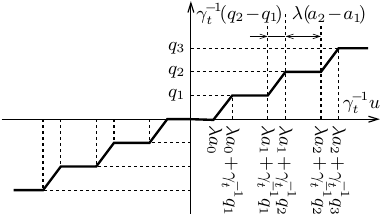}
}%
\captionbox{Asymptotic scaled map as $\gamma_t\to 0$.
\label{fig:prox-scale-inf}}[0.5\linewidth]{
\centering
\includegraphics[width=0.85\linewidth]{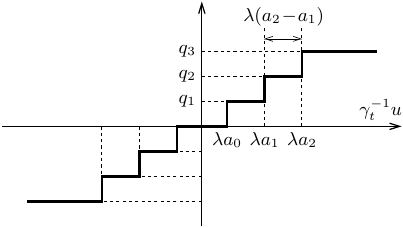}
}
\end{figure}

\subsection{AProx versus Prox-SGD and ProxQuant}
To better understand the difference between AProx and Prox-SGD, we rewrite Prox-SGD in~\eqref{eqn:prox-sgd} as
\begin{equation}\label{eqn:prox-sgd-2-step}
\begin{aligned}
u^{t+1} &= \textcolor{blue}{w^t} - \eta_t g^t, \\[-1ex]
w^{t+1} &= \prox_{\textcolor{blue}{\eta_t}\lambda\Psi}(u^{t+1}),
\end{aligned}
\end{equation}
which differ from AProx in~\eqref{eqn:aprox} in two places (highlighted in blue). Here we give an intuitive interpretation of these differences. 
First notice that the objective in~\eqref{eqn:erm+reg} is the sum of~$f$ and~$\lambda\Psi$, and both methods make progress by using the stochastic gradient of $f$ (forward step) and the proximal map of $\lambda\Psi$ (backward step) --- in a balanced manner.
\vspace{-1ex}
\begin{itemize}\itemsep 0pt
\item
In Prox-SGD, $u^{t+1}$ is a combination of $w^t$ and $-\eta_t g^t$. But $w^t$ already contains contributions from both $f$ and $\lambda\Psi$, through $\{-\eta_s g^s\}_{s=1}^{t-1}$ and $\{\prox_{\eta_s\lambda\Psi}\}_{s=1}^{t-1}$ respectively.
Therefore, from $u^{t+1}$ to obtain $w^{t+1}$, we should use $\prox_{\eta_t\lambda\Psi}$ to balance $-\eta_t g^t$.
\item
For AProx, $u^{t+1}$ is used to accumulate $\sum_{s=1}^{t}\eta_s g^s$, solely contributed from~$f$.
Thus in computing $w^{t+1}$, we need to strike a balance with the contribution from $\lambda\Psi$ with the aggregated strength $\gamma_t=\sum_{s=1}^t\eta_s$.
\end{itemize}
\vspace{-1ex}

While the total contributions from the forward steps ($-\eta_t g^t$) and backward steps ($\prox_{\lambda\Psi}$) are balanced in both cases, Prox-SGD spreads the backward steps on every iterate $w_t$
so the quantization effect on the last iterate eventually diminishes.
In contrast, AProx always applies an aggregate proximal map to generate the last iterate, 
in order to balance the accumulation of pure forward steps in ~$u^{t+1}$. 


\citet{dockhorn2021demystifying}
used the regularization framework and proximal maps to demystify BinaryConnect/STE and developed a generalization called ProxConnect.
It is derived from the generalized conditional gradient method \citep{yu2017gcd}, through the machinery of Fenchel-Rockafellar duality.
We derived AProx as an direct extension of RDA \citep{xiao2010rda}, but realized that it is indeed equivalent to ProxConnect, with some minor differences in setting $\gamma_t$. 
Nevertheless, our construction through balancing forward and backward steps provides a more intuitive understanding of the algorithm and may shed light on further development of structure-inducing optimization algorithms.


\subsection{Convergence Analysis}


To simplify the presentation, we define
\[
F_{\lambda}(w):=\E_z[f(w,z)]+\lambda \Psi(w).
\]
The following theorem concerns the convergence of AProx in terms of the weighted average
\(\bar{w}^t = \frac{1}{\sum_{s=1}^t\eta_s}\sum_{s=1}^t \eta_s w^s\).
This result has appeared in \citet[Cor.~5.2.]{dockhorn2021demystifying}. 
We include it here as a basis for proving last-iterate convergence and its proof in Appendix~\ref{sec:appendix-proof} for completeness.

\begin{theorem}\label{thm:prox-frl-stoch}
Assume that $f(w,z)$ is convex in~$w$ for any~$z$, $\Psi$ is convex, and $F_{\lambda}$ is continuous with Lipschitz constant~$G$. Also, let
$\mathcal{W}^\star$ be the set minimizers of $F_{\lambda}(w)$.
Then,
\begin{enumerate}[label=(\alph*),ref=(\alph*),nosep]
\item  
If the stepsize $\eta_t$ satisfies~\eqref{eqn:stepsize-cond} and $\{w_s\}_{s=1}^t$ are generated by algorithm~\eqref{eqn:aprox},
then the weighted average $\bar{w}^t$
converges in expectation to a point in $\mathcal{W}^\star$.
\item 
Let $w^0$ be an initial point, 
$R=\!\min_{w^\star\in\mathcal{W}^\star}\!\|w^0-w^\star\|_2$
and the step size $\eta_t=\frac{R}{2G}\sqrt{\frac{1}{t}}$, then
\vspace{-1ex}
\[
\E\bigl[F_{\lambda}(\bar{w}^t)\bigr]-F_{\lambda}(w^\star)
~\leq~ GR\frac{2+1.5\ln(t)}{\sqrt{t}},
\]
where the expectation $\E[\cdot]$ is taken with respect to the sequence of random variables $\{w^1,\ldots,w^t\}$.
\end{enumerate}
\end{theorem}





While convergence results on averaged iterates are conventional in the stochastic optimization literature, they are far from satisfactory for QAT. 
In particular, the averaged iterates $\bar{w}^t$ are mostly \emph{not} quantized even if every $w^t$ is quantized. 

In general, last-iterate convergence of stochastic/online algorithms is crucial for regularized optimization problems aiming for a structured solution (such as sparsity and quantization). 
Here we provide such a result for AProx.

\begin{theorem}[Last-iterate convergence of AProx for convex optimization]
\label{thm::last-iterate-convergence}
Under the same assumptions as in Theorem~\ref{thm:prox-frl-stoch}, the last iterate $w^t$ of AProx satisfies
\[
    \E\left[F_{\lambda}(w^t)\right]-F_{\lambda}(w^*)\leq GR\frac{2+1.5\ln(t)}{\sqrt{t}}.
\]
\end{theorem}
The proof of Theorem~\ref{thm::last-iterate-convergence} is provided in Appendix~\ref{sec:appendix-proof-last-iterate-convergence}. We note that this convergence rate matches the average-iterate convergence rate established in Theorem~\ref{thm:prox-frl-stoch}. 


\begin{figure}[t]
\captionsetup[subfigure]{aboveskip=10pt}
\centering
\begin{subfigure}[b]{0.5\linewidth}
\begin{center}
\includegraphics[width=0.55\linewidth]{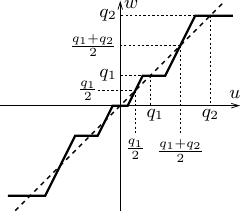}
\end{center}
\caption{$\prox_\text{PARQ}(u,\mathcal{Q},\rho)$}
\label{fig:prox-parq}
\end{subfigure}%
\begin{subfigure}[b]{0.5\linewidth}
\begin{center}
\includegraphics[width=0.55\linewidth]{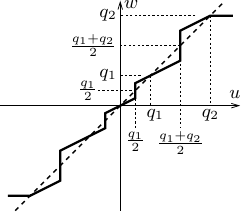}
\end{center}
\caption{$\prox_\text{BinRel}(u,\mathcal{Q},\rho)$}
\label{fig:prox-binrel}
\end{subfigure}%
\caption{Proximal maps of PARQ and BinaryRelax.}
\label{fig:prox-impl}
\end{figure}

\section{PARQ: A Practical Implementation}
\label{sec:implementation}

A practical issue for implementing AProx is how to choose the PAR parameters $\{q_k\}_{k=1}^m$ and $\{a_k\}_{k=0}^{m-1}$, as well as the regularization strength~$\lambda$; 
see their roles in the proximal map
in Figure~\ref{fig:prox-scale-gamma}.
In particular, $\{q_k\}$ are the target quantization values for $w^{t}$ and $\lambda$ and $\{a_k\}$ determine the quantization thresholds on the scaled input $\gamma_t^{-1} u^t$.
In practice, it is very hard to choose these parameters a priori for different models and datasets. 
Therefore, we propose a heuristic approach to estimate the target values $\{q_k\}$ online and at the same time avoid setting $\lambda$ and $\{a_k\}$ explicitly.

Given a vector $u^t\in\R^d$, we need to quantize it (element-wise) to a vector $w^t\in\mathcal{Q}^d$ where $w^t_i\in\mathcal{Q}$ for $i=1,\ldots,d$.
We use the least-squares binary quantization (LSBQ) approach~\cite{pouransari2020lsb} to estimate the target quantization values in~$\mathcal{Q}$.
LSBQ employs a form of $n$-bit \emph{scaled binary quantization}, i.e., let
\(w_i = \sum_{j=1}^n v_j s_j(u_i)\)
where each $v_j\in\R_+$ satisfies $v_1\geq\cdots\geq v_n\geq 0$ and each $s_j:\R\to\{-1,1\}$ is a binary function.  
The optimal $\{v_j,s_j(\cdot)\}_{j=1}^n$ for approximating $u\in\R^d$ in the least-squares sense can be found by solving the problem:
\[
\begin{array}{ll}
\minimize_{\{v_j,s_j(\cdot)\}} & \sum_{i=1}^d \bigl(u_i - \sum_{j=1}^n v_j s_j(u_i)\bigr)^2 \\
\mbox{subject to} & v_1\geq v_2 \geq \cdots \geq v_n \geq 0, \\
& s_j:\R\to\{-1,1\} , ~j=1,\ldots,n.
\end{array}
\]
For $n=1$ (1-bit quantization), the solution is well-known: 
$v_1=\|u\|_1/d$ and $s_1(u_i)=\sgn(u_i)$ 
\citep[e.g.,][]{rastegari2016xnor}.
\citet{pouransari2020lsb} derived the solutions for the $n=2$ case and the ternary case ($n=2$ with $v_1=v_2$);
see also \citet{yin2019lowbit}.
For $n>2$, there is no closed-form solution, but \citet{pouransari2020lsb} gives a simple greedy algorithm for \emph{foldable} representations, which satisfy 
$s_j(u_i)=\sgn(u_i-\sum_{\ell=1}^{j-1}v_{\ell} s_{\ell}(u_i))$ for all $j=1,\ldots,n$.

Once a set of (exact or approximate) solution $\{v_j\}_{j=1}^n$ is obtained, the resulting quantization values can be written in the form
$\pm v_1 \pm\cdots \pm v_n$ by choosing either~$+$ or~$-$ between the adjacent operands. For example, 
the largest and smallest values in~$\mathcal{Q}=\{\pm q_1,\ldots,\pm q_m\}$ are \(q_m = v_1 + \cdots + v_n\) and \(-q_m = -v_1 - \cdots - v_n\).
Since there are~$n$ binary bits, the total number of target values is $|\mathcal{Q}|=2^n$.

The selection of $\{a_k\}$ and~$\lambda$ is somewhat arbitrary and not consequential. 
We can choose them so that the asymptotic graph in Figure~\ref{fig:prox-scale-inf} matches the hard-quantization map depicted in Figure~\ref{fig:quant-map}. That is, we can let 
$\lambda a_k = (q_k+q_{k+1})/2$, but never really use them once $\mathcal{Q}$ is found by LSBQ.

While in theory we require $\gamma_t=\sum_{s=1}^t\eta_s\to +\infty$, in practice it can only reach a not-very-large constant due to a finite number of iterations we run with diminishing step sizes. Therefore its effect on scaling the horizontal axis in Figures~\ref{fig:prox-scale-gamma} and~\ref{fig:prox-scale-inf} is limited and can be absorbed by tuning the step size.
On the other hand, we would like the proximal map to be able to converge to hard-quantization by the end of training (so we have fully quantized solutions). 
For this purpose, we use an independent schedule for growing the slope of the slanted segments. 
Specifically, we emulate the proximal map in Figure~\ref{fig:prox-scale-gamma} with the one in Figure~\ref{fig:prox-parq}, where $\mathcal{Q}$ is calculated from LSBQ, and $\rho$ is the slope of the slanted segments. 
For convenience, we specify a schedule for the \emph{inverse slope} $\rho_t^{-1}$ to vary monotonically from~$1$ to~$0$ during $T$ steps of training (so the slope $\rho_t$ go to infinity).
For example, $\rho_t^{-1}$ can follow a cosine decay schedule, or one in the sigmoid family as shown at the bottom of Figure~\ref{fig:deit-results}.

Putting everything together, we have PARQ in Algorithm~\ref{alg:parq}.

\begin{algorithm}[t]
\caption{PARQ}
\label{alg:parq}
\begin{algorithmic}[0]
\STATE \textbf{input:} $w^1\in\R^d$, 
number of quantization bits $n$,
\STATE \qquad\quad step sizes $\{\eta_t\}_{t=1}^T$, slope schedule $\{\rho_t^{-1}\}_{t=1}^T$
\STATE \textbf{initialize:} $u^1=w^1$
\FOR{$t=1,2,\ldots,T\!-\!1$}
\STATE $u^{t+1} = u^t \!- \eta_t\, \nabla f(w^t, z^t)$
\STATE $\mathcal{Q}^{t+1} = \text{LSBQ}(u^{t+1}, n)$
\STATE $w^{t+1} = \prox_{\text{PARQ}}(u^{t+1}, \mathcal{Q}^{t+1}, \rho_t) $
\ENDFOR
\STATE \textbf{output:} $w^{T}$
\end{algorithmic}
\end{algorithm}

\section{Experiments}
\label{sec:experiments}

We train convolutional and vision-transformer models on image classification tasks across five bit-widths: ternary (T) and 1--4 bits. For each model and bit-width pair, we compare PARQ to two QAT methods: STE/BinaryConnect \citep{courbariaux2015binaryconnect} and BinaryRelax \citep{yin2018binaryrelax}. 

Specifically,
STE/BinaryConnect uses the hard-quantization map in Figure~\ref{fig:quant-map}, 
PARQ applies the proximal map in Figure~\ref{fig:prox-parq} with slope annealing, 
and BinaryRelax effectively uses the one in Figure~\ref{fig:prox-binrel} where the slope of slanted segments gradually decreases to $0$.
We note that $\prox_{\text{PARQ}}$ is the proximal map of a convex PAR, but STE and $\prox_{\text{BinRel}}$ do not correspond to convex regularization.

Each entry in Tables~\ref{tab:resnet20_56-results}--\ref{tab:deit-results} shows the mean and standard dev-iation of test accuracies over three randomly seeded runs.
Full-precision (FP) accuracy is shown in parentheses under each model depth/size.

We provide an open-source PyTorch package 
\url{https://github.com/facebookresearch/parq},
which implements PARQ and several other popular QAT methods and can reproduce the results presented in this section.

\begin{table}[t]
\centering
{
\small
\begin{tabular}{ccccc}\toprule
Depth & \# bits & STE & BinaryRelax & PARQ\\
\midrule
\multirow{5}{*}{\shortstack{20\\\\(91.82)}} & 1 & 89.56 \stderr{0.18} & 89.98 \stderr{0.13} & {\bftab 90.48} \stderr{0.26}\\
& T & 90.94 \stderr{0.15} & 91.25 \stderr{0.07} & {\bftab 91.45} \stderr{0.11}\\
& 2 & 91.22 \stderr{0.15} & 91.57 \stderr{0.06} & 91.71 \stderr{0.03}\\
& 3 & 91.84 \stderr{0.22} & 91.77 \stderr{0.05} & 91.97 \stderr{0.04}\\
& 4 & 91.93 \stderr{0.04} & 91.92 \stderr{0.16} & 91.84 \stderr{0.02}\\
\midrule

\multirow{5}{*}{\shortstack{56\\\\(93.08)}} & 1 & 91.55 \stderr{0.33} & 91.75 \stderr{0.37} & 91.34 \stderr{0.37}\\
& T & 92.42 \stderr{0.09} & 92.34 \stderr{0.23} & {\bftab 92.97} \stderr{0.15}\\
& 2 & 92.72 \stderr{0.27} & 92.30 \stderr{0.40} & 92.77 \stderr{0.10}\\
& 3 & 92.73 \stderr{0.44} & 92.86 \stderr{0.40} & 92.45 \stderr{0.44}\\
& 4 & 92.34 \stderr{0.23} & 92.59 \stderr{0.10} & 92.49 \stderr{0.16}\\
\bottomrule
\end{tabular}
}
\caption{ResNet test accuracy on CIFAR-10.}
\label{tab:resnet20_56-results}
\end{table}

\subsection{ResNet on CIFAR-10}
\label{sec:cifar10}

We first evaluate quantized ResNet-20 and ResNet-56 \cite{he2016deep} on CIFAR-10. 
All weights, including those in the final projection layer, are quantized. We train for 200 epochs using SGD with \num{0.9} momentum and \num{2e-4} weight decay. Following \citet{zhu2022trained}, the \num{0.1} learning rate decays by a factor of 10 at epochs 80, 120, and 150.

As shown in Table~\ref{tab:resnet20_56-results},
PARQ performs competitively to STE and BinaryRelax across all bit-widths. For 1-bit ResNet-20, it outperforms STE by nearly one accuracy point. It is the only QAT method for ternary ResNet-56 reaching within \num{\sim 0.1} points of full-precision accuracy. 

\subsection{ResNet on ImageNet}
\label{sec:resnet-imagenet}

\begin{table}[b]
\centering
{
\small
\begin{tabular}{ccccc}\toprule
Depth & \# bits & STE & BinaryRelax & PARQ\\
\midrule
\multirow{5}{*}{\shortstack{50\\\\(75.60)}} & 1 & 66.17 \stderr{0.04} & 66.14 \stderr{0.28} & {\bftab 66.71} \stderr{0.13}\\
& T & 70.94 \stderr{0.19} & 71.59 \stderr{0.11} & 71.35 \stderr{0.21}\\
& 2 & 72.38 \stderr{0.10} & {\bftab 72.64} \stderr{0.17} & 72.43 \stderr{0.03}\\
& 3 & 73.58 \stderr{0.09} & 74.02 \stderr{0.09} & 73.91 \stderr{0.13}\\
& 4 & 74.52 \stderr{0.04} & 74.58 \stderr{0.04} & 74.52 \stderr{0.01}\\
\bottomrule
\end{tabular}
}
\caption{Quantized ResNet-50 test accuracy on ImageNet.}
\label{tab:resnet50-results}
\end{table}

For QAT of ResNet-50 \cite{he2016deep} on ImageNet, we quantize all residual block weights per channel by computing $\mathcal{Q}$ row-wise over tensors. We use SGD with \num{0.1} learning rate, \num{0.9} momentum, and \num{1e-4} weight decay. The learning rate decays by a factor of 10 every 30 epochs.

Similar to the experiments on CIFAR-10, PARQ performs capably against STE and BinaryRelax in Table~\ref{tab:resnet50-results}. It shows a slight advantage in the most restrictive 1-bit case, achieving a half-point margin over the other two methods.

\subsection{DeiT on ImageNet}
\label{sec:deit-imagenet}

\begin{table}[t]
\centering
{
\small
\begin{tabular}{ccccc}\toprule
Size & \# bits & STE & BinaryRelax & PARQ\\
\midrule
\multirow{5}{*}{\shortstack{Ti\\\\(71.91)}} & 1 & 51.62 \stderr{0.18} & 52.62 \stderr{0.03} & 52.51 \stderr{0.19}\\
& T & 61.43 \stderr{0.08} & {\bftab 62.18} \stderr{0.11} & 60.99 \stderr{0.07}\\
& 2 & 64.81 \stderr{0.15} & 65.20 \stderr{0.04} & {\bftab 65.32} \stderr{0.06}\\
& 3 & 69.02 \stderr{0.11} & 69.26 \stderr{0.03} & {\bftab 69.47} \stderr{0.04}\\
& 4 & 70.95 \stderr{0.11} & 71.06 \stderr{0.09} & 71.21 \stderr{0.11}\\
\midrule

\multirow{5}{*}{\shortstack{S\\\\(79.80)}} & 1 & 70.07 \stderr{0.03} & 70.69 \stderr{0.07} & {\bftab 71.06} \stderr{0.02}\\
& T & 75.83 \stderr{0.06} & 76.02 \stderr{0.03} & {\bftab 76.30} \stderr{0.06}\\
& 2 & 77.40 \stderr{0.01} & 77.43 \stderr{0.04} & {\bftab 77.63} \stderr{0.04}\\
& 3 & 79.02 \stderr{0.14} & 79.11 \stderr{0.07} & 79.04 \stderr{0.04}\\
& 4 & 79.57 \stderr{0.04} & 79.55 \stderr{0.12} & 79.61 \stderr{0.04}\\
\midrule

\multirow{5}{*}{\shortstack{B\\\\(81.73)}} & 1 & 78.79 \stderr{0.03} & 79.02 \stderr{0.03} & {\bftab 79.35} \stderr{0.04} \\
& T & 80.50 \stderr{0.01} & 80.61 \stderr{0.08} & 80.62 \stderr{0.01}\\
& 2 & 80.73 \stderr{0.17} & 80.81 \stderr{0.14} & 80.84 \stderr{0.06}\\
& 3 & 80.54 \stderr{0.20} & {\bftab 80.94} \stderr{0.05} & 80.59 \stderr{0.12}\\
& 4 & 80.45 \stderr{0.10} & {\bftab 80.76} \stderr{0.12} & 80.35 \stderr{0.19}\\

\bottomrule
\end{tabular}
}
\vskip 0.15in
\caption{Quantized DeiT test accuracy on ImageNet.}
\label{tab:deit-results}
\end{table}

Applying QAT to a different architecture, we experiment with Data-efficient image Transformers \citep[DeiT]{touvron2021training}. Our DeiT experiments include the Ti, S, and B model sizes with 5M, 22M, and 86M parameters, respectively. Attention block weights are quantized channel-wise as in Section~\ref{sec:resnet-imagenet}. Embeddings, layer normalization parameters, and the final projection weights are left at full precision, following the setting of \citet{rastegari2016xnor}.

We use AdamW \cite{loshchilov2018decoupled} to train for 300 epochs with a \num{5e-4} learning rate and \num{0.05} weight decay. We hold the learning rate at \num{1e-8} for the final 20 epochs (after PARQ and BinaryRelax converge to hard-quantization); this boosts performance relative to the default \num{1e-5} minimum. We apply RandAugment \cite{cubuk2020randaugment} and all prior regularization strategies \cite{zhang2017mixup,yun2019cutmix} except repeated augmentation \cite{berman2019multigrain}.

We observe in Table~\ref{tab:deit-results} that PARQ's performance trends stay true across model sizes. For 1-bit DeiT-S, PARQ improves upon STE accuracy by a full point. 
Figure~\ref{fig:deit-results} shows the training loss curves of three different QAT methods along with full precision (FP) training on the DeiT-Ti model.
We observe that in the initial phase, PARQ closely follows the FP curve because the slope of the slanted segments in its proximal map (Figure~\ref{fig:prox-parq}) is close to~1. 
Then the training loss of PARQ increases due to the relatively sharp transition of the slope, and it follows the STE curve closely as its proximal map converges to hard quantization.
Figure~\ref{fig:quant-map-deit} gives snapshots of how PAR gradually induces quantization in model parameters: compare the middle stage plot with Figure~\ref{fig:prox-parq}) and the late stage plot with Figure~\ref{fig:quant-map}.

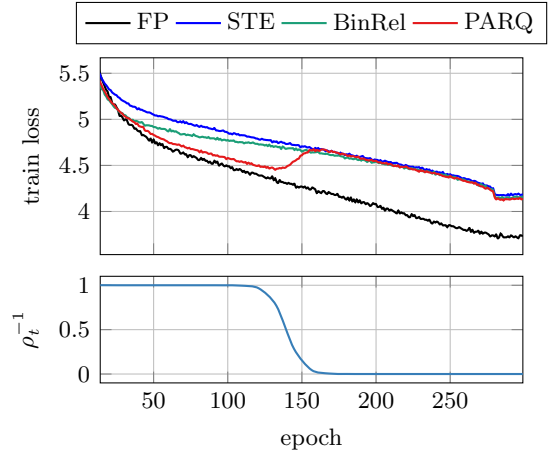
\begin{figure}[t]
\centering
\begin{tikzpicture}[font=\footnotesize]
\begin{groupplot}[
    group style={
        group size=1 by 2,
        vertical sep=0.75em,
        x descriptions at=edge bottom,
    },
    every axis plot/.append style={thick,smooth},
    width=7.2cm,
    enlarge x limits=false,
    grid,xlabel=epoch
]
    \nextgroupplot[xmin=14,ylabel=train loss,legend to name={deit-legend},legend style={legend columns=4},height=4.2cm]
    \addplot[black] table[x=epoch,y=train_loss,col sep=comma] {data/tiny_fp.csv};
    \addplot[blue] table[x=epoch,y=train_loss,col sep=comma] {data/tiny_2bit_hard.csv};
    \addplot[Teal] table[x=epoch,y=train_loss,col sep=comma] {data/tiny_2bit_binrel.csv};
    \addplot[Red] table[x=epoch,y=train_loss,col sep=comma] {data/tiny_2bit_parq.csv};
    \legend{FP,STE,BinRel,PARQ}

    \nextgroupplot[height=3cm,ylabel={$\rho_t^{-1}$}]
    \addplot[Blue,domain=14:299] {
        (1/(1 + e^(50 * (min(x / 279, 1) - 0.5))) - 1/(1 + e^(0.5 * x))) / (1/(1+e^(-0.5 * x)) - 1/(1 + e^(0.5 * x)))
    };
\end{groupplot}
\path (group c1r1.north east) -- node[inner sep=0,yshift=1.2em]{\pgfplotslegendfromname{deit-legend}} (group c1r1.north west);
\end{tikzpicture}
\vspace{-2ex}
\caption{Training loss curves for 2-bit DeiT-Ti model (top) and the inverse-slope schedule $\rho_t^{-1}$ used by PARQ (bottom).}
\label{fig:deit-results}
\end{figure}

\begin{figure*}[t!]
\setlength{\fullwidth}{\textwidth}
\centering
\begin{tikzpicture}[font=\footnotesize]
\begin{groupplot}[
    group style={group size=3 by 1,horizontal sep=3em,y descriptions at=edge left},
    every axis plot/.append style={only marks,mark=x,mark size=1pt,},
    every axis title/.append style={yshift=-0.5em},
    width=4.5cm,height=3.825cm,
    grid,
    xmin=-0.4,xmax=0.4,
    ymin=-0.2,ymax=0.2,
    xlabel={$u$},ylabel={$w$},
]
\nextgroupplot[title={epoch 6}]
\addplot[only marks,mark=x,mark size=1pt] table[col sep=comma] {data/tiny_2bit_005.csv};
\nextgroupplot[title={epoch 150}]
\addplot[only marks,mark=x,mark size=1pt] table[col sep=comma] {data/tiny_2bit_149.csv};
\nextgroupplot[title={epoch 300}]
\addplot[only marks,mark=x,mark size=1pt] table[col sep=comma] {data/tiny_2bit_299.csv};
\end{groupplot}
\end{tikzpicture}
\vspace{-2ex}
\caption{PARQ proximal maps during early, middle, and late stages of training 2-bit DeiT-Ti (value weights from an attention layer).}
\label{fig:quant-map-deit}
\end{figure*}
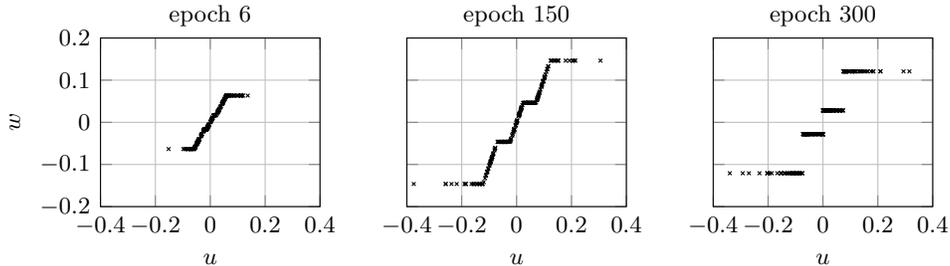

\newpage

\section{Conclusion}
\label{sec:conclusion}

We developed a principled approach for quantization-aware training (QAT) based on a framework of convex, piecewise-affine regularization (PAR).
In order to avoid the diminishing regularization effect of the standard proximal SGD method, we derive an aggregate proximal (AProx) algorithm and prove its last-iterate convergence.



Our experiments demonstrate that PARQ achieves competitive performance compared with QAT methods that correspond to using nonconvex regularization.
Compared with using hard-quantization (STE) throughout the training process, the gradual evolution of PARQ from piecewise-affine soft quantization to hard quantization helps the training process to be more stable, and often converges to better local minima. This is more evident in the most demanding cases of low-bit quantization of smaller models.




%




\newpage
\appendix
\onecolumn

\section{Convergence analysis}
\subsection{Proof of Theorem~\ref{thm:prox-frl-stoch}}
\label{sec:appendix-proof}


We consider the framework of online convex optimization, which is more general than stochastic optimization. 
In particular, let $f_t=f(\cdot,z^t)$ be a function presented to us at each iteration~$t=1,2,\ldots$, and $\Psi$ be a regularization function that we use throughout the whole process.
The two-step presentation of AProx in~\eqref{eqn:aprox} can be written in one-step as 
\begin{equation}\label{eqn:prox-frl}
w^{t+1} = \argmin_{w\in \cW} ~\biggl\{ \sum_{s=1}^t \eta_s\bigl(\langle g^s, w\rangle + \lambda\Psi(w)\bigr) + \frac{1}{2}\|w-w^0\|_2^2 \biggr\},
\end{equation}
where $w^0$ is the initial weight vector and $g^t=\nabla f_t(w^t)$.
Moreover, we use a more general distance generating function~$h$ to replace $(1/2)\|\cdot\|_2^2$, and define the Bregman divergence as
\[
D_h(u,w) = h(u) - h(w) - \langle \nabla h(w), u-w\rangle.
\]
With Bregman divergence, a more general form of AProx can be written as
\begin{equation}\label{eqn:prox-frl-general}
w^{t+1} = \argmin_{w\in\cW} ~\biggl\{ \sum_{s=1}^t \bigl(\eta_s\langle g^s, w\rangle + \lambda \Psi(w) \bigr) + D_h(w,w^0) \biggr\}.
\end{equation}
\begin{assumption}\label{asmp:aprox}
We make the following assumptions:
\begin{enumerate}[label=(\alph*),ref=(\alph*)]
\item Each loss function $f_t$ is convex and Lipschitz continuous with Lipschitz constant $G_f$.
\label{asmp:f-lip}
\item The regularizer $\Psi$ is convex and Lipschitz continuous with Lipschitz constant $G_\Psi$.
\label{asmp:Psi-lip}
\item The function $h$ is differentiable and strongly convex with convexity parameter~$\rho$.
\label{asmp:h-sc}
\end{enumerate}
\end{assumption}

It follows from Assumption~\ref{asmp:aprox}\ref{asmp:h-sc} that $D_h(u,w)$ is strongly convex in~$w$ with convexity parameter~$\rho$.

\bigskip

\begin{theorem}[Regret bound for AProx]
\label{thm:prox-frl-general}
Under Assumption~\ref{asmp:aprox}, for any $w\in\R^d$, it holds that 
\begin{equation}\label{eqn:frl-regret-bound}
\sum_{s=1}^t \eta_s\bigl(f_s(w^s)+\lambda\Psi(w^s) - f_s(w)-\lambda\Psi(w)\bigr)
\leq \frac{(G_f+\lambda G_\Psi)^2}{\rho} \sum_{s=1}^t 2\eta_s^2 + D_h(w,w^0).
\end{equation}
\end{theorem}

\begin{proof}
We adapt the proof of \citet[Theorem~4.3]{bubeck2015cvxopt} by adding the regularizer~$\Psi$ and replacing the term $h(w)-h(w^0)$ with $D_h(w,w^0)$. An advantage of this replacement is that we can use any initial point~$w^0$ while the proof in \cite{xiao2010rda,bubeck2015cvxopt} requires $w^0=\argmin h(w)$.

Let $w^0\in\R^d$ be an arbitrary initial point and define $\psi_0(w)=D_h(w,w^0)$. 
For $t\geq 1$, define
\[
\psi_t(w) := \sum_{s=1}^t \eta_s\bigl(\langle g^s,w\rangle + \lambda \Psi(w)\bigr) + D_h(w, w^0).
\]
The AProx algorithm~\eqref{eqn:prox-frl-general} can be expressed as, for $t\geq 0$, 
\[
w^{t+1} = \argmin_{w}~\psi_t(w).
\]
Since $D_h(w, w^0)$ is strongly convex in~$w$ with convexity parameter~$\rho$, the same property holds for~$\psi_t$ for all $t\geq 0$.
According to a basic result on minimizing strongly convex functions
\citep[e.g.,][Lemma~3.2]{chenteboulle1993} and the fact that $w^{t+1}$ minimizes $\psi_t$, we have
\begin{equation}\label{eqn:sc-min-gap}
\psi_t(w^{t+1}) \leq \psi_t(w) - \frac{\rho}{2}\|w-w^{t+1}\|^2, 
\qquad t=0,1,2,\ldots.
\end{equation}
From the definition of $\psi_t$ and $\psi_{t-1}$, we have
\begin{equation}\label{eqn:psi-decomp}
\psi_t(w^t) - \psi_t(w^{t+1})
= \psi_{t-1}(w^t) - \psi_{t-1}(w^{t+1}) + \eta_t\bigl(\langle g^t, w^t-w^{t+1}\rangle + \lambda \Psi(w^t)- \lambda \Psi(w^{t+1})\bigr).
\end{equation}
For the left-hand side of~\eqref{eqn:psi-decomp}, we apply~\eqref{eqn:sc-min-gap} to obtain
\[
\frac{\rho}{2}\|w^{t+1}-w^t\|^2 \leq \psi_t(w^t)-\psi_t(w^{t+1}).
\]
For the first term on the right-hand side of~\eqref{eqn:psi-decomp}, we apply~\eqref{eqn:sc-min-gap} again for $\psi_{t-1}$ to obtain
\[
\psi_{t-1}(w^t)-\psi_{t-1}(w^{t+1}) \leq - \frac{\rho}{2}\|w^{t+1}-w^t\|^2 .
\]
For the second term on the right-hand side of~\eqref{eqn:psi-decomp}, we have
\begin{align}
\langle g^t, w^t-w^{t+1}\rangle + \lambda \Psi(w^t)- \lambda \Psi(w^{t+1})
&\leq \|g^t\|_*\|w^{t+1}-w^t\| + \lambda \Psi(w^t)- \lambda \Psi(w^{t+1})
\nonumber \\
&\leq G_f\|w^{t+1}-w^t\| + +\lambda G_\Psi\|w^{t+1}-w^t\| \nonumber \\
&= (G_f+\lambda G_\Psi)\|w^{t+1}-w^t\|,
\label{eqn:lip-bound}
\end{align}
where in the first inequality we used H\"older's inequality, and in the second inequality we used Assumptions~\ref{asmp:aprox}\ref{asmp:f-lip} and~\ref{asmp:aprox}\ref{asmp:Psi-lip} respectively.
Combining the above three inequalities with~\eqref{eqn:psi-decomp}, we get
\[
\rho\|w^{t+1}-w^t\|^2 \leq \eta_t (G_f+\lambda G_\Psi)\|w^{t+1}-w^t\|,
\]
which further implies
\[
\|w^{t+1}-w^t\| \leq \eta_t (G_f+\lambda G_\Psi)/\rho.
\]
Combining this with~\eqref{eqn:lip-bound} yields
\begin{equation}\label{eqn:inner-bound}
\langle g^t, w^t-w^{t+1}\rangle + \lambda \Psi(w^t)- \lambda \Psi(w^{t+1})
\leq \eta_t (G_f+\lambda G_\Psi)^2 /\rho.
\end{equation}

Next we prove that the following inequality holds for all $w\in\R^d$ and all $t\geq 0$:
\begin{equation}\label{eqn:inc-bound}
\sum_{s=1}^t \eta_s\bigl(\langle g^s, w^{s+1}\rangle + \lambda \Psi(w^{s+1})\bigr) 
\leq \sum_{s=1}^t \eta_s\bigl(\langle g^s, w\rangle + \lambda \Psi(w)\bigr) + D_h(w,w^0). 
\end{equation}
We proceed by induction. For the base case $t=0$, the desired inequality becomes $D_h (w,w^0)\geq 0$, which is always true by the definition of~$D_h$.
Now we suppose~\eqref{eqn:inc-bound} holds for $t-1$ and apply it with $w=w^{t+1}$ in the first inequality below:
\begin{align*}
&\sum_{s=1}^t \eta_s\bigl(\langle g^s, w^{s+1}\rangle + \lambda \Psi(w^{s+1})\bigr) \\
&= \sum_{s=1}^{t-1} \eta_s\bigl(\langle g^s, w^{s+1}\rangle + \lambda \Psi(w^{s+1})\bigr) + \eta_t\bigl(\langle g^t, w^{t+1}\rangle + \lambda \Psi(w^{t+1})\bigr) \\
&\leq \sum_{s=1}^{t-1} \eta_s\bigl(\langle g^s, w^{t+1}\rangle + \lambda \Psi(w^{t+1})\bigr) + D_h(w^{t+1},w^0) + \eta_t\bigl(\langle g^t, w^{t+1}\rangle + \lambda \Psi(w^{t+1})\bigr) \\
&=\sum_{s=1}^t \eta_s\bigl(\langle g^s, w^{t+1}\rangle + \lambda \Psi(w^{t+1})\bigr) + D_h(w^{t+1},w^0)  \\
&\leq \sum_{s=1}^t \eta_s\bigl(\langle g^s, w\rangle + \lambda \Psi(w)\bigr) + D_h(w,w^0), \qquad \forall\,w\in\mathcal{W}.
\end{align*}
In the last inequality above, we recognized the definition of $\psi_t$ and used the fact that $w^{t+1}$ is the minimizer of~$\psi_t$.
This finishes the proof of~\eqref{eqn:inc-bound}.

Finally we add $\sum_{s=1}^t\eta_s\left(\langle g^s,w^s\rangle+\Psi(w^s)\right)$ to both sides of~\eqref{eqn:inc-bound} and rearrange terms to obtain
\begin{align}
\sum_{s=1}^t \eta_s\bigl(\langle g^s, w^s-w\rangle + \lambda \Psi(w^s)- \lambda \Psi(w)\bigr) 
\leq &
\sum_{s=1}^t \eta_s\bigl(\langle g^s, w^s-w^{s+1}\rangle + \lambda \Psi(w^s)- \lambda \Psi(w^{s+1})\bigr) + D_h(w,w^0). 
\label{eqn:delta-bound}
\end{align}
For the left-hand side of~\eqref{eqn:delta-bound}, we use convexity of~$f_s$ to obtain
\[
f_s(w^s)-f_s(w) \leq \langle g^s, w^s - w\rangle. 
\]
For the right-hand side of~\eqref{eqn:delta-bound}, we apply~\eqref{eqn:inner-bound} to obtain
\[
\sum_{s=1}^t \eta_s\bigl(\langle g^s, w^s-w^{s+1}\rangle + \lambda \Psi(w^s)- \lambda \Psi(w^{s+1})\bigr) 
\leq \frac{(G_f+\lambda G_\Psi)^2}{\rho}\sum_{s=1}^t \eta_s^2.
\]
Combining the above three inequalities together, we have
\[
\sum_{s=1}^t \eta_s\bigl(f_s(w^s)+\Psi(w^s) - f_s(w)- \lambda \Psi(w)\bigr)
\leq \frac{(G_f+\lambda G_\Psi)^2}{\rho}\sum_{s=1}^t \eta_s^2 + D_h(w,w^0).
\]
This finishes the proof of Theorem~\ref{thm:prox-frl-general}.
\end{proof}

Now we consider the stochastic optimization problem of minimizing 
$f(w) + \lambda \Psi(w)$ where the loss function $f(w):=\E_z[f(w,z)]$.
We can regard the sequence of loss functions $f_t$ in the online optimization setting as $f(\cdot,z^t)$ and compare with $w^\star=\argmin~f(w)+\lambda \Psi(w)$.
In this case, the regret bound~\eqref{eqn:frl-regret-bound} becomes
\[
\sum_{s=1}^t \eta_s\bigl(f(w^s,z^s)+\lambda \Psi(w^s) - f(w^\star,z^s)- \lambda \Psi(w^\star)\bigr)
\leq \frac{(G_f+\lambda G_\Psi)^2}{\rho}\sum_{s=1}^t \eta_s^2 + D_h(w^\star,w^0).
\]
Using a standard online-to-stochastic conversion argument \citep[e.g.,][Theorem~3]{xiao2010rda}, we can derive
\begin{equation}
    \sum_{s=1}^t \eta_s\left(\E\bigl[f(w^s)+\lambda \Psi(w^s)\bigr] - f(w^\star)- \lambda \Psi(w^\star)\right)
\leq \frac{(G_f+\lambda G_\Psi)^2}{\rho}\sum_{s=1}^t \eta_s^2 + D_h(w^\star,w^0),
\label{eq::average-convergence}
\end{equation}
where the expectation $\E[\cdot]$ is taken with respect to the random variables $\{w^1,\ldots,w^t\}$, which in turn depends on $\{z^1,\ldots,z^t\}$.

For the ease of presentation, we denote $R^2=\min_{w\in\mathcal{W}}D_h(w,w^0)$.
Moreover, we define a weighted average of all iterates up to iteration~$t$:
\[
\bar{w}^t = \frac{1}{\sum_{s=1}^t\eta_s}\sum_{s=1}^t \eta_s w^s.
\]
Then by convexity of~$f$ and $\Psi$, we obtain
\begin{equation}\label{eqn:stoch-bound}
\E\bigl[f(\bar{w}^t)+\lambda \Psi(\bar{w}^t)\bigr]-f(w^\star)- \lambda \Psi(w^\star)\leq 
\frac{\frac{(G_f+\lambda G_\Psi)^2}{\rho}\sum_{s=1}^t \eta_s^2 + R^2}{\sum_{s=1}^t\eta_s}.
\end{equation}

\paragraph{Constant stepsize.} 
If the total number of iterations~$T$ is known ahead of time, then we can choose an optimal constant stepsize. 
Let $\eta_s=\eta$ for all $s=1,\ldots,T$, then the bound in~\eqref{eqn:stoch-bound} becomes
\[
\frac{\frac{(G_f+\lambda G_\Psi)^2}{\rho}T\eta^2 + R^2}{T\eta}
=\frac{(G_f+\lambda G_\Psi)^2}{\rho} \eta + \frac{R^2}{T\eta}.
\]
In order to minimize the above bound, we take $\eta=\frac{R}{G_f+\lambda G_\Psi}\sqrt{\frac{\rho}{T}}$ and obtain
\[
\E\bigl[f(\bar{w}^T)+\lambda \Psi(\bar{w}^T)\bigr]-f(w^\star)- \lambda \Psi(w^\star)
~\leq~ 2(G_f+\lambda G_\Psi)R\sqrt{\frac{1}{\rho\, T}}.
\]

\paragraph{Diminishing stepsize.}
The right-hand side of~\eqref{eqn:stoch-bound} has the same form as the convergence rate bound for the classical stochastic gradient or subgradient method \citep[e.g.,][Section~3.2.3]{nesterov2004book}.
A classical sufficient condition for convergence is 
\[
\sum_{s=1}^\infty \eta_s = +\infty, 
\qquad
\sum_{s=1}^\infty \eta_s^2 < +\infty.
\]
In particular, if we take $\eta_t=\frac{R}{2(G_f+\lambda G_\Psi)}\sqrt{\frac{\rho}{t}}$, we have 
\[
\E\bigl[f(\bar{w}^t)+\lambda \Psi(\bar{w}^t)\bigr]-f(w^\star)- \lambda \Psi(w^\star)
~\leq~ (G_f+\lambda G_\Psi)R \frac{(2+1.5\ln(t))}{\sqrt{\rho t}}.
\]


Finally, Theorem~\ref{thm:prox-frl-stoch} is obtained with some simplification. 
In particular, if we choose the Bregman divergence as the Euclidean distance $\frac{1}{2}\|\cdot\|_2^2$, then we have $\rho=1$. This leads to
\[
\E\bigl[f(\bar{w}^t)+\lambda \Psi(\bar{w}^t)\bigr]-f(w^\star)- \lambda \Psi(w^\star)
~\leq~ G R \frac{(2+1.5\ln(t))}{\sqrt{t}},
\]
where $G:=G_f+\lambda G_\Psi$. This completes the proof.

\subsection{Proof of Theorem~\ref{thm::last-iterate-convergence}}
\label{sec:appendix-proof-last-iterate-convergence}
For simplicity, we denote $F_{\lambda}(w)=f(w)+\lambda\Psi(w)$ and $G=G_f+\lambda G_\Psi$ where $G_f$ and $G_\Psi$ are the Lipschitz constants of $f$ and $\Psi$, respectively.

To establish the last-iterate convergence of AProx, we first introduce the following lemma, which connects the convergence of the last iteration to the convergence of the average iteration.
\begin{lemma}[Lemma~1 in \cite{Orabona_2020}]
\label{lem::last-averge-convertion}
    Given that $\{\eta_t\}_{t=1}^T$ is a non-increasing positive sequence and $\{q_t\}_{t=1}^T$ is a nonnegative sequence, the following inequality holds 
    \begin{equation}
        \eta_T q_T \leqslant \frac{1}{T} \sum_{t=1}^T \eta_t q_t+\sum_{k=1}^{T-1} \frac{1}{k(k+1)} \sum_{t=T-k+1}^T \eta_t\left(q_t-q_{T-k}\right).
    \end{equation}
\end{lemma}
Upon setting $q_t=\E\left[F_{\lambda}(w^t)\right]-F_{\lambda}(w^*)$ in Lemma~\ref{lem::last-averge-convertion}, we derive that  
\begin{equation}
        \begin{aligned}
            \eta_T \left(\E\left[F_{\lambda}(w^T)\right]-F_{\lambda}(w^*)\right) &\leq \frac{1}{T} \sum_{t=1}^T \eta_t \left(\E\left[F_{\lambda}(w^t)\right]-F_{\lambda}(w^*)\right) \\
            &\quad +\sum_{k=1}^{T-1} \frac{1}{k(k+1)} \sum_{t=T-k+1}^T \eta_t\E\left[F_{\lambda}(w^t)-F_{\lambda}(w^{T-k})\right].
        \end{aligned}
    \end{equation}
For the first term on the right-hand side, we apply Equation~\ref{eq::average-convergence}, which yields 
\begin{equation}
    \frac{1}{T} \sum_{t=1}^T \eta_t \left(\E\left[F_{\lambda}(w^t)\right]-F_{\lambda}(w^*)\right)\leq \frac{G^2}{\rho T}\sum_{t=1}^{T}\eta_t^2+\frac{D_h(w^*, w^0)}{T}.
\end{equation}
To control the second term, we note that for any $1\leq k\leq T-1$
\begin{equation}
    \begin{aligned}
        \sum_{t=T-k+1}^T \eta_t\E\left[F_{\lambda}(w^t)-F_{\lambda}(w^{T-k})\right]&=\sum_{t=T-k}^T \eta_t\E\left[F_{\lambda}(w^t)-F_{\lambda}(w^{T-k})\right]\leq \frac{G^2}{\rho}\sum_{t=T-k}^{T}\eta_t^2.
    \end{aligned}
\end{equation}
Here we apply Equation~\ref{eq::average-convergence} again for the last inequality upon setting $w^\star=w^{T-k}$ and use the fact that $D_h(w, w)=0$ for all $w\in \cW$. 

Combining the above two components together, we have 
\begin{equation}   
\label{eq::appendix-last-iterate}
\E\left[F_{\lambda}(w^T)\right]-F_{\lambda}(w^*)\leq \frac{G^2}{\eta_T\rho}\left(\frac{1}{T}\sum_{t=1}^T\eta_t^2+\sum_{k=1}^{T-1}\frac{1}{k(k+1)}\sum_{t=T-k}^T\eta_t^2\right)+\frac{D_h(w^*, w^0)}{\eta_T T}.
\end{equation}

\paragraph{Constant stepsize.} 
If the total number of iterations~$T$ is known ahead of time, then we can choose an optimal constant stepsize. 
Let $\eta_t=\eta$ for all $s=1,\ldots,T$, then the bound in~\eqref{eq::appendix-last-iterate} becomes
\begin{equation}            \E\left[F_{\lambda}(w^T)\right]-F_{\lambda}(w^*)\leq \frac{G^2}{\rho}\left(1+\sum_{k=1}^{T-1}\frac{1}{k}\right)\eta+\frac{D_h(w^*, w^0)}{\eta T}\leq \frac{G^2}{\rho}\left(2+\ln(T)\right)\eta+\frac{D_h(w^*, w^0)}{\eta T}.
\end{equation}
Here we use the fact that $\sum_{k=1}^n\frac{1}{k}\leq 1+\ln(n)$ for all $n\geq 1$.
In order to minimize the above bound, we take $\eta=\frac{1}{G}\sqrt{\frac{D_h(w^*, w^0)\rho}{(2+\ln(T))T}}$ and obtain
\begin{equation}
    \E\left[F_{\lambda}(w^T)\right]-F_{\lambda}(w^*)\leq 2G\sqrt{\frac{D_h(w^*, w^0)(2+\ln(T))}{\rho T}}.
\end{equation}

\paragraph{Diminishing stepsize.}
Suppose we set the stepsize $\eta_t=\frac{\eta}{\sqrt{t}}$. Then, Equation~\ref{eq::appendix-last-iterate} reduces to 
\begin{equation}
    \begin{aligned}
        \E\left[F_{\lambda}(w^T)\right]-F_{\lambda}(w^*)&\leq \frac{\eta \sqrt{T}G^2}{\rho}\left(\frac{1}{T}\sum_{t=1}^T\frac{1}{t}+\sum_{k=1}^{T-1}\frac{1}{k(k+1)}\sum_{t=T-k}^T\frac{1}{t}\right)+\frac{D_h(w^*, w^0)}{\eta \sqrt{T}}\\
    &\leq \frac{\eta \sqrt{T}G^2}{\rho}\left(\frac{1+\ln(T)}{T}+\sum_{k=1}^{T-1}\frac{1}{k(k+1)}\sum_{t=T-k}^T\frac{1}{t}\right)+\frac{D_h(w^*, w^0)}{\eta \sqrt{T}}.
    \label{eq::diminishing}
    \end{aligned}
\end{equation}
To proceed, note that 
\begin{equation}
    \sum_{t=T-k+1}^T \frac{1}{t} \leq \int_{T-k}^T \frac{1}{t} d t=\ln \left(\frac{T}{T-k}\right)=\ln \left(1+\frac{k}{T-k}\right) \leq \frac{k}{T-k}.
\end{equation}
Therefore, we have 
\begin{equation}
\begin{aligned}
\sum_{k=1}^{T-1} \frac{1}{k(k+1)} \sum_{t=T-k}^T \frac{1}{t}
& =\sum_{k=1}^{T-1} \frac{1}{k(k+1)}\left(\frac{1}{T-k}+\sum_{t=T-k+1}^T \frac{1}{t}\right) \\
& \leq \sum_{k=1}^{T-1} \frac{1}{k(T-k)} \\
& =\sum_{k=1}^{T-1} \frac{1}{k T}+\sum_{k=1}^{T-1} \frac{1}{T(T-k)} \\
& = 2\sum_{k=1}^{T-1} \frac{1}{k T}\\
& \leq 2 \frac{1+\ln (T)}{T}.
\end{aligned}
\end{equation}
Invoking this result into Equation~\ref{eq::diminishing}, we further have 
\begin{equation}
    \E\left[F_{\lambda}(w^T)\right]-F_{\lambda}(w^*)\leq \frac{3\eta G^2(1+\ln(T))}{\rho \sqrt{T}}+\frac{D_h(w^*, w^0)}{\eta\sqrt{T}}.
\end{equation}
Hence, upon setting $\eta=\frac{1}{G}\sqrt{\frac{D_h(w^*, w^0)\rho}{2}}$, we derive that
\begin{equation}
    \E\left[F_{\lambda}(w^T)\right]-F_{\lambda}(w^*)\leq G\left(2\sqrt{2}+\frac{3}{\sqrt{2}}\ln(T)\right)\sqrt{\frac{D_h(w^*, w^0)}{\rho T}}.
\end{equation}
Specifically, if we choose the Bregman divergence as the Euclidean distance $\frac{1}{2}\|\cdot\|_2^2$, then we have $\rho=1$. Upon defining $R=\min_{w^\star\in\mathcal{W}^\star}\|w^0-w^\star\|_2$, we have 
\begin{equation}
    \E\left[F_{\lambda}(w^T)\right]-F_{\lambda}(w^*)\leq GR\frac{(2+\frac{3}{2}\ln(T))}{\sqrt{T}}.
\end{equation}

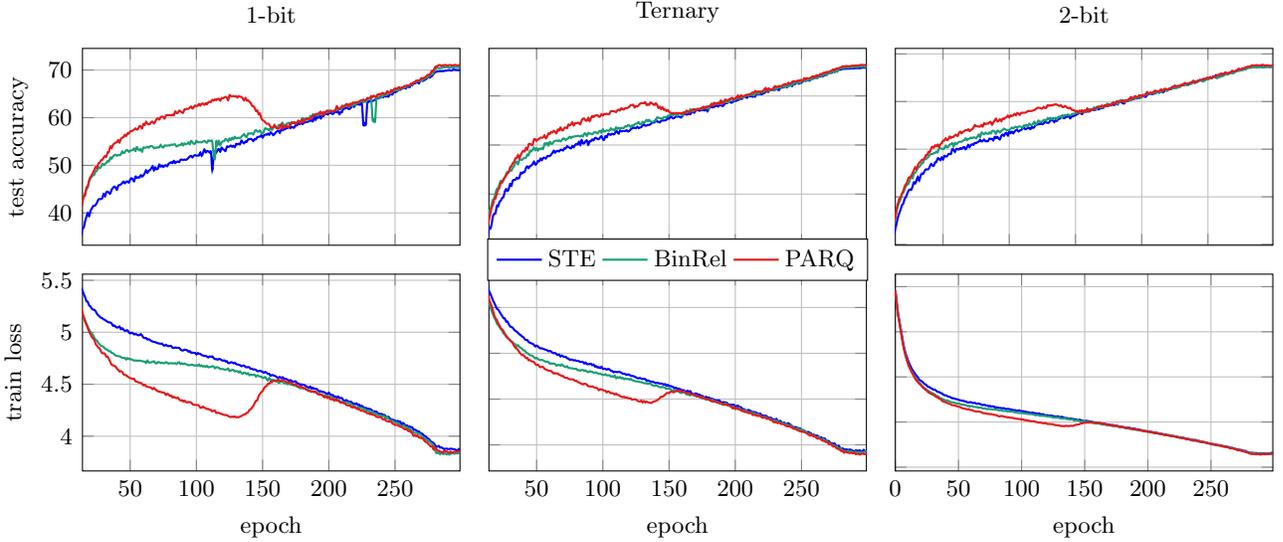
\begin{figure}[t]
\centering
\begin{tikzpicture}[font=\footnotesize]
\begin{groupplot}[
    group style={
        group size=3 by 2,
        x descriptions at=edge bottom,
        horizontal sep=1em,
        vertical sep=1em,
    },
    every axis plot/.append style={thick,smooth},
    width=6.6cm,height=4.2cm,
    enlarge x limits=false,
    grid,xlabel=epoch
]
    \nextgroupplot[title={1-bit},xmin=14,ylabel={test accuracy}]
    \addplot[blue] table[x=epoch,y=test_acc,col sep=comma] {data/small_1bit_hard.csv};
    \addplot[Teal] table[x=epoch,y=test_acc,col sep=comma] {data/small_1bit_binrel.csv};
    \addplot[Red] table[x=epoch,y=test_acc,col sep=comma] {data/small_1bit_parq.csv};

    \nextgroupplot[title={Ternary},xmin=14,yticklabels=\empty]
    \addplot[blue] table[x=epoch,y=test_acc,col sep=comma] {data/small_tern_hard.csv};
    \addplot[Teal] table[x=epoch,y=test_acc,col sep=comma] {data/small_tern_binrel.csv};
    \addplot[Red] table[x=epoch,y=test_acc,col sep=comma] {data/small_tern_parq.csv};

    \nextgroupplot[title={2-bit},xmin=14,yticklabels=\empty]
    \addplot[blue] table[x=epoch,y=test_acc,col sep=comma] {data/small_2bit_hard.csv};
    \addplot[Teal] table[x=epoch,y=test_acc,col sep=comma] {data/small_2bit_binrel.csv};
    \addplot[Red] table[x=epoch,y=test_acc,col sep=comma] {data/small_2bit_parq.csv};

    \nextgroupplot[xmin=14,ylabel={train loss},legend=\empty]
    \addplot[blue] table[x=epoch,y=train_loss,col sep=comma] {data/small_1bit_hard.csv};
    \addplot[Teal] table[x=epoch,y=train_loss,col sep=comma] {data/small_1bit_binrel.csv};
    \addplot[Red] table[x=epoch,y=train_loss,col sep=comma] {data/small_1bit_parq.csv};
    
    \nextgroupplot[xmin=14,yticklabels=\empty,legend to name={curve-legend}, legend style={legend columns=3}]
    \addplot[blue] table[x=epoch,y=train_loss,col sep=comma] {data/small_tern_hard.csv};
    \addplot[Teal] table[x=epoch,y=train_loss,col sep=comma] {data/small_tern_binrel.csv};
    \addplot[Red] table[x=epoch,y=train_loss,col sep=comma] {data/small_tern_parq.csv};
    \legend{STE,BinRel,PARQ}

    \nextgroupplot[yticklabels=\empty]
    \addplot[blue] table[x=epoch,y=train_loss,col sep=comma] {data/small_2bit_hard.csv};
    \addplot[Teal] table[x=epoch,y=train_loss,col sep=comma] {data/small_2bit_binrel.csv};
    \addplot[Red] table[x=epoch,y=train_loss,col sep=comma] {data/small_2bit_parq.csv};
\end{groupplot}
\path (group c2r2.north east) -- node[yshift=0.5em]{\pgfplotslegendfromname{curve-legend}} (group c2r2.north west);
\end{tikzpicture}
\caption{DeiT-S test accuracy (top row) and train loss (bottom row) across several bit-widths (columns).}
\label{fig:deit-all-results}
\end{figure}

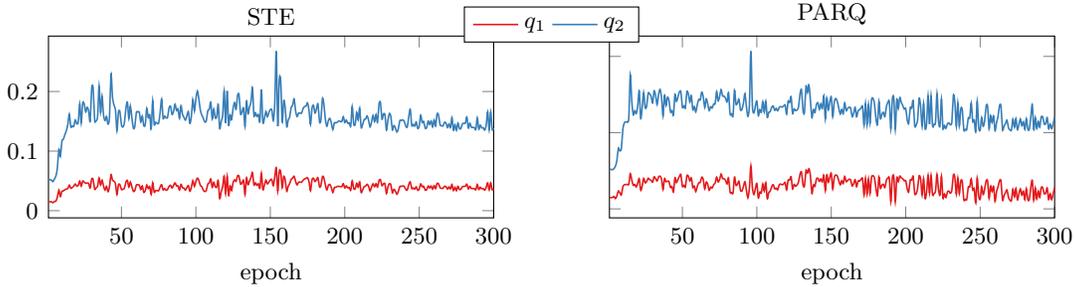
\begin{figure}[ht]
\centering
\begin{tikzpicture}[font=\footnotesize]
\begin{groupplot}[
    group style={
        group size=2 by 1,
        yticklabels at=edge left,
        horizontal sep=4em,
    },
    every axis plot/.append style={semithick,smooth},
    every axis title/.append style={yshift=-0.5em},
    width=7.5cm,height=4cm,
    enlarge x limits=false,
    xlabel={epoch},
]
    \nextgroupplot[title={STE}]
    \addplot+ table[col sep=comma,x=x,y=q3]{data/tiny_2bit_ste_q.csv};
    \addplot+ table[col sep=comma,x=x,y=q4]{data/tiny_2bit_ste_q.csv};

    \nextgroupplot[title={PARQ},legend to name={shared-legend}, legend style={legend columns=4}]
    \addplot+ table[col sep=comma,x=x,y=q3]{data/tiny_2bit_parq_q.csv};
    \addplot+ table[col sep=comma,x=x,y=q4]{data/tiny_2bit_parq_q.csv};
    \legend{$q_1$,$q_2$}
\end{groupplot}
\path (group c1r1.north east) -- node[yshift=0.4em]{\pgfplotslegendfromname{shared-legend}} (group c2r1.north west);
\end{tikzpicture}
\caption{Evolution of $\{q_1,q_2\}$ (estimated by LSBQ) during training of a 2-bit DeiT-Ti model.}
\label{fig:quant-q-evol}
\end{figure}

\section{Additional experiment results}

\label{sec:appendix-experiments}

Figure~\ref{fig:deit-all-results} presents
accuracy and training loss curves for QAT of DeiT-S.
The top left plot 
reveals that PARQ often has stabler training dynamics.  It does not suffer from the sudden accuracy drops seen in STE and BinaryRelax. This could be due to PARQ's more gradual annealing in the first half of training. It performs the most consistently on DeiT-S, suggesting the relative performance of QAT methods may vary by model size.

Figure~\ref{fig:quant-q-evol} shows the evolution of $\{q_1,q_2\}$ (estimated by LSBQ) during training of a 2-bit DeiT-Ti model. They are from the same layer as the one used in Figure~\ref{fig:quant-map-deit} and with the same weight initialization.
It is clear that magnitudes of $\mathcal{Q}$ start small from randomly initialized weights, expand rapidly in early stages of training, then slowly contract in later epochs.

\newpage

\bibliographystyle{icml2025}
\bibliography{parq}

\end{document}